\documentclass{article}

\usepackage{microtype}
\usepackage{amsmath}
\usepackage{bbm}
\usepackage{algorithm}
\usepackage[noend]{algorithmic}
\usepackage{amssymb}
\usepackage{amsthm}
\usepackage{hyperref}
\usepackage{thmtools}
\usepackage{thm-restate}
\usepackage{enumitem}
\usepackage[numbers]{natbib}
\usepackage[left=1in,right=1in,top=1in,bottom=1in]{geometry}
\usepackage[T1]{fontenc}
\usepackage{authblk}
\usepackage{color}

\newcommand{\set}[1]{\left\{ #1 \right\}}

\DeclareMathOperator*{\argmax}{arg\,max}

\DeclareMathOperator*{\polylog}{polylog}

\newcommand{\abs}[1]{\left| #1 \right|}
\newcommand{\R}{\mathbb{R}}
\newcommand{\Prob}[1]{{\mathbb{P}\left[{#1}\right]}}

\newcommand{\E}[1]{{\mathbb{E}\left[{#1}\right]}}

\newcommand{\ee}[2]{\mathcal{E}_{#1}\left[ #2 \right]}
\newcommand{\expo}[1]{\exp\left( #1 \right)}
\newcommand{\parens}[1]{\left( #1 \right)}
\newcommand{\bracks}[1]{\left[ #1 \right]}
\newcommand{\mathand}[0]{\quad\textrm{and}\quad}

\newcommand{\bbN}[0]{\mathbb{N}}

\newcommand{\laplace}[1]{\textrm{Laplace}\parens{#1}}

\newcommand{\D}[0]{\mathcal{D}}

\newcommand{\M}[0]{\mathcal{M}}
\newcommand{\N}[0]{\mathcal{N}}

\newcommand{\X}[0]{\mathcal{X}}
\newcommand{\Q}[0]{\mathcal{Q}}
\newcommand{\ExD}[1]{\underset{x\sim\D}{\mathbb{E}}\left[#1\right]}

\newcommand{\blockremove}[1]{  }
\newtheorem{definition}{Definition}
\newtheorem{theorem}{Theorem}
\newtheorem{lemma}{Lemma}

\newtheorem{proposition}{Proposition}

\title{The Everlasting Database: Statistical Validity at a Fair Price}
\author[1]{Blake Woodworth%\thanks{
}
\author[2]{Vitaly Feldman}
\author[3]{Saharon Rosset}
\author[1]{Nathan Srebro}
\affil[1]{Toyota Technological Institute at Chicago}
\affil[2]{Google Brain}
\affil[3]{School of Mathematical Sciences, Tel Aviv University}

\date{}

\begin{document}
\maketitle
\begin{abstract}
  The problem of handling adaptivity in data analysis, intentional or not,  permeates
  a variety of fields, including  test-set overfitting in ML challenges and the
  accumulation of invalid scientific discoveries.
  We propose a mechanism for answering an arbitrarily long sequence of
  potentially adaptive statistical queries, by charging a price for
  each query and using the proceeds to collect additional samples.
  Crucially, we guarantee statistical validity without any assumptions on
  how the queries are generated. We also ensure with high probability that
  the cost for $M$ non-adaptive queries is $O(\log M)$,
  while the cost to a potentially adaptive user who makes $M$
  queries that do not depend on any others is $O(\sqrt{M})$.
\end{abstract}
\section{Introduction}
Consider the problem of running a server that provides the test loss of a
model on held out data, e.g.~for evaluation in a machine learning challenge.
 We would like to ensure that all test losses returned
by the server are accurate estimates of the true generalization error of
the predictors.

Returning the empirical error on held out test data
would initially be a good estimate of the generalization error.
However, an analyst can use the empirical errors to adjust
their model and improve their performance on the test data.  In
fact, with a number of queries only linear in the amount of
test data, one can easily create a predictor that completely
overfits, having empirical error on the test data that
is artificially small \cite{DworkFHPRR15:science,BlumH15}.
Even without such intentional overfitting, sequential querying can lead to unintentional adaptation
since analysts are biased toward
tweaks that lead to improved test errors.

If the queries were non-adaptive, i.e.~the sequence of predictors
is not influenced by previous test
results, then we could handle
a much larger number of queries before overfitting--a number exponential in the size of the dataset. Nevertheless,
the test set will eventually be ``used up'' and estimates of the test error (specifically those of the best performers)
might be over-optimistic.

A similar situation arises in other contexts such as validating potential scientific
discoveries. 
One can evaluate potential discoveries using set aside validation data,
but if analyses are refined adaptively based on the results, one may again overfit the
validation data and arrive at false discoveries \cite{ioannidis2005most,GelmanLoken13}.

One way to ensure the validity of answers in the face of adaptive querying
is to collect all queries before giving
any answers, and answer them all at once, e.g.~at the
end of a competition.  However, analysts typically want more
immediate feedback, both for ML challenges and in scientific
research.  Additionally, if we want to answer more
queries later, ensuring statistical validity would require collecting a whole new dataset.  This might be
unnecessarily expensive if few or none of the queries are in fact adaptive. It also raises the question of who should bear the cost of collecting new data.

Alternatively, we could try to limit the
number or frequency of queries from each user, forbid adaptive
querying, or assume users work
independently of each other, remaining oblivious to other users' queries
and answers.  However, it is nearly impossible to enforce such restrictions.
Determined users can avoid querying restrictions by creating spurious
user accounts and working in groups; there is no feasible way to check
if queries are chosen adaptively; and
information can leak between analysts, intentionally or not,
e.g.~through explicit collaboration or published results.

In this paper, we address the fundamental challenge of providing statistically valid answers to
an arbitrarily long sequence of potentially adaptive queries.
% Naturally, this would not be possible
% using a dataset of fixed size and
We assume that it is possible to collect additional samples from
the same data distribution at a fixed cost per sample. To pay for new samples, users of
the database will be charged for their queries.
%we allow the dataset
%We assume that additional data samples
%can be obtained for a fixed cost, and that users would be willing to pay a small (and vanishing) cost for use of a database with strong guarantees.
We propose a mechanism, \textsc{EverlastingValidation},
that guarantees ``everlasting'' statistical validity and maintains the following properties:
\begin{description}[noitemsep,topsep=0pt]
\item[Validity] Without any assumptions about the users,
  and even with arbitrary adaptivity,
  with high probability, all answers ever returned by the database are
  accurate.
\item[Self-Sustainability] The database collects enough revenue to
  purchase as many new samples as necessary in perpetuity, and can
  answer an \emph{unlimited} number of queries.
\item[Cost for Non-Adaptive Users] With high probability, a user making $M$
  non-adaptive queries will pay at most $O(\log M)$, so the average
  cost per query decreases as $\tilde{O}(1/M)$.
\item[Cost for Autonomous Users] With high probability, a user (or group of users) making
  $M$ potentially adaptive queries that depend on each other
  arbitrarily, but not on any queries made by others, will pay at most
  $\tilde{O}(\sqrt{M})$, so the average cost per query decreases
  as $\tilde{O}(1/\sqrt{M})$.
\end{description}
We emphasize that the database mechanism needs no notion of ``user''
or ``account'' when answering the queries; it does not need to know
which ``user'' made which query; and most of all, it does not need to
know whether a query was made adaptively or not.
Rather, the cost guarantees hold for any collection of queries
that are either non-adaptive or autonomous in the sense described
above--a ``user'' could thus refer to a single individual, or if
an analyst uses answers from another person's queries, we can
consider them together as an ``autonomous user'' and get cost guarantees
based on their combined number of queries. The database's cost guarantees
are nearly optimal; the cost to non-adaptive users and the cost to autonomous users cannot
be improved (beyond log-factors) while still maintaining validity and sustainability
(Section \ref{sec:optimality}).

As is indicated by the guarantees above, using the mechanism adaptively
may be far more expensive than using it non-adaptively.
We view this as a
positive feature.  Although we cannot enforce non-adaptivity, and it is
sometimes unreasonable to expect that analysts are entirely non-adaptive,
we intend the mechanism to be used for {\em validation}.  That
is, analysts should do their discovery, training, tuning, development,
and adaptive data analysis on unrestricted
``training'' or ``discovery'' datasets, and only use the protected
database when they wish to receive a stamp of approval on their model,
predictor, or discovery.  Instead of trying to police or forbid adaptivity,
we discourage it with pricing, but in a way that is essentially guaranteed not to affect
non-adaptive users. Further, users will need to pay a high price only when their queries
explicitly cause overfitting, so only adaptivity that is harmful
to statistical validity will be penalized.

\paragraph{Relationship to prior work} 
Our work is inspired by
% a recent line of work on ensuring statistical validity for adaptively-chosen queries
% (e.g. \cite{DworkFHPRR15:science,DworkFHPRR14:arxiv,HardtU14,SteinkeU15,BlumH15,dwork2015generalization,BassilyNSSSU16,RussoZ16,RogersRST16,FeldmanS17,Hardt17arxiv}).
a number of mechanisms for dealing with potentially adaptive queries that have been proposed and analyzed using techniques from differential privacy and information theory. These mechanisms handle only a pre-determined number of
queries using a fixed dataset.  We use techniques developed in this literature, in particular addition of
noise to ensure that a quadratically larger number of adaptive queries can be answered in the worst case
\cite{DworkFHPRR14:arxiv,BunSteinke16}.
Our main innovations over this prior work are the self-sustaining
nature of the database, as opposed to handling only a pre-determined
number of queries of each type, and also the per-query pricing scheme
that places the cost burden on the adaptive users. To ensure that the cost burden on non-adaptive users does not grow by more than a constant factor, we need to adapt existing algorithms. 

\textsc{Ladder} \cite{BlumH15} and \textsc{ShakyLadder}
\cite{Hardt17arxiv} are mechanisms tailored to maintaining a
ML competition leaderboard.  These algorithms reveal the answer to a user's query for the error of their model only if it is significantly lower than the error of the previous best submission from the user. While these mechanisms can handle an exponential number of
arbitrarily adaptive submissions, each user will receive answers to a relatively small number of queries. Our setting is more suitable for the case where we want to validate the errors of all submissions or for scientific discovery where there is more then one discovery to be made.
% At the same time, those mechanisms can easily be combined with our approach to making validity guarantees everlasting by collecting additional data and penalizing only adaptive users.

A separate line of work in the statistics literature on ``Quality Preserving Databases'' (\citet{Aharoni2014QPD} and references therein) has suggested schemes for databases that maintain everlasting validity, while charging for use. The fundamental difference from our work is that these schemes do not account for adaptivity and thus are limited to non-adaptive querying. A second difference is that they focus on hypothesis testing for scientific discovery, with pricing schemes that depend on considerations of statistical power, which are not part of our framework. We further compare with existing methods at the end of Section \ref{sec:everlastingvalidation}.

\section{Model formulation}
We consider a setting in which a database curator has access to samples
from some unknown distribution $\D$ over a sample space $\X$. 
% Multiple analysts
% interact with the database by posing queries about $\D$.
Multiple analysts submit a sequence of statistical queries
$q_i:\X\rightarrow[0,1]$, the database responds with answers $a_i \in
\R$, and the goal is to ensure that with high probability, all answers
satisfy $\abs{a_i - \mathbb{E}_{x\sim\D}\left[q_i(x)\right]} \leq
\tau$ for some fixed accuracy parameter $\tau$.  In a prediction
validation application, each query would measure the expected loss
of a particular model, while in scientific applications a single query
might measure the value of some phenomenon of interest, or compare it to a ``null'' reference.
We denote $\Q$ the set of all possible queries, i.e.~measurable
functions $q:\X\rightarrow[0,1]$, and use the shorthand
$\E{q}=\mathbb{E}_{x\sim\D}\left[q(x)\right]$ to denote the mean value
(desired answer) for each query. Given a data sample
$S\sim\D^n$, we use $\ee{S}{q} = \frac{1}{\abs{S}}\sum_{x\in S}q(x)$
as shorthand for the empirical mean of $q$ on $S$.

In our framework, the database can, at any time, acquire
new samples from $\D$ at some fixed cost per sample, e.g.~by running
more experiments or paying workers to label
more data. To answer a given query, the database can use the samples
it has already purchased in any way it chooses, and
the database is allowed to charge analysts for their queries in order
to purchase additional samples. The price $p_i$ of query $q_i$ may
be determined by the database after it receives query $q_i$, allowing
the database to charge more for queries that force it to collect more data.

We do not assume the queries are chosen in advance, and instead allow the sequence
of queries to depend adaptively on past answers.  More formally, we
define a ``querying rule'' $R_i: (\Q,\R,\R)^{i-1} \mapsto \Q$ as a
randomized mapping from the history of all previously made queries and
their answers and prices to the statistical query to be made next:
\begin{equation*}
q_i = R_i\left((q_1,a_1,p_1),(q_2,a_2,p_2),\ldots,(q_{i-1},a_{i-1},p_{i-1})\right).
\end{equation*}
The interaction of users with the database can then be modeled as a
sequence of querying rules $\set{R_i}_{i\in\bbN}$.  The combination of
the data distribution, database mechanism, and sequence of
querying rules together define a joint distribution over queries,
answers, and prices $\set{Q_i,A_i,P_i}_{i\in\bbN}$.  All our results
will hold for any data distribution and any querying sequence, with
high probability over $\set{Q_i,A_i,P_i}_{i\in\bbN}$.

We think of the query sequence as representing a combination of
queries from multiple users, but the database itself is unaware of the
identity or behavior of the users.  Our validity guarantees do not assume any particular user
structure, nor any constraints on the interactions of the different
users. Thus, the guarantees are always valid regardless of
what a ``user'' means, how ``users'' are allowed to collaborate, how
many ``users'' there are, or how many queries each ``user'' makes---the
guarantees simply hold for any (arbitrarily adaptive) querying
sequence.

% However, our cost guarantees will, and must, refer to analysts (or perhaps
% groups of analysts) behaving in specific ways.  In particular, we
% define an {\bf autonomous user} of the database as a subsequence
% $\set{u_j}_{j\in[M]}$ of the querying rules such that all querying rules
% within the subsequence depend only on the
% history \emph{within the subsequence}, i.e.
% \begin{multline*}
% R_{u_j}\left((q_1,a_1,p_1),\ldots,(q_{(u_j - 1)},a_{(u_j -
%     1)},p_{(u_j-1)})\right) = \\
% R_{u_j}\left((q_{u_1},a_{u_1},p_{u_1}),\ldots,(q_{u_{(j-1)}},a_{u_{(j-1)}},p_{u_{(j-1)}})\right).
% \end{multline*}
% That is, $Q_{u_j}$ is independent of the overall past history given
% the past history pertaining to the autonomous user.  We further define a
% \textbf{non-adaptive user} as a subsequence consisting of the queries which do not
% depend on \emph{any} of the history, i.e.
% $R_{u_j}\left((q_1,a_1,p_1),\ldots,(q_{(u_j - 1)},a_{(u_j -
%     1)},p_{(u_j-1)})\right)$ is a fixed (pre-determined) distribution
% over queries, and so $Q_{u_j}$ is independent of all of the history. The
% cost to a user (total price paid for queries) is $\sum_{j=1}^M
% p_{u_j}$.

However, our cost guarantees will, and must, refer to analysts (or perhaps
groups of analysts) behaving in specific ways. In particular, we define a
\textbf{non-adaptive user} as a subsequence $\set{u_j}_{j\in[M]}$ 
consisting of queries which do not depend on \emph{any} of the history, 
% i.e.~$R_{u_j}\left((q_1,a_1,p_1),\ldots,(q_{(u_j - 1)},a_{(u_j -
%     1)},p_{(u_j-1)})\right)$ is a fixed (pre-determined) distribution
i.e.~$R_{u_j}$ is a fixed (pre-determined) distribution
over queries, so $Q_{u_j}$ is independent of all of the history. We further define
an {\bf autonomous user} of the database as a subsequence
$\set{u_j}_{j\in[M]}$ of the querying rules that depend only on the history \emph{within the subsequence}, i.e.
\begin{equation*}
R_{u_j}\left((q_1,a_1,p_1),\ldots,(q_{(u_j - 1)},a_{(u_j -
    1)},p_{(u_j-1)})\right) =
R_{u_j}\left((q_{u_1},a_{u_1},p_{u_1}),\ldots,(q_{u_{(j-1)}},a_{u_{(j-1)}},p_{u_{(j-1)}})\right).
\end{equation*}
That is, $Q_{u_j}$ is independent of the overall past history given
the past history pertaining to the autonomous user. The ``cost to a user''
is the total price paid for queries in the subsequence $\set{u_j}$: $\sum_{j=1}^M p_{u_j}$.

\section{\textsc{ValidationRound}}\label{sec:validationround}
Our mechanism for providing ``everlasting'' validity guarantees is based on a query answering mechanism which we call \textsc{ValidationRound}. It uses $n$ samples from $\D$ in order to answer $\exp(\Omega(n))$ non-adaptive and at least $\tilde{\Omega}(n^2)$ adaptive statistical queries (and potentially many more). Our analysis is based on ideas developed in the context of adaptive data analysis \cite{DworkFHPRR14:arxiv} and relies on techniques from differential privacy \cite{DworkMNS:06}. Differential privacy is a strong stability property of randomized algorithms that operate on a dataset. Composition properties of differential privacy imply that this form of stability holds even when the same dataset is used by multiple algorithms that can depend on the outputs of preceding algorithms. Most importantly, differential privacy implies generalization with high probability \cite{DworkFHPRR14:arxiv,BassilyNSSSU16}.

\textsc{ValidationRound} splits its data into two sets $S$ and $T$. Upon receiving each query, it first checks whether the answers on these datasets approximately agree. If so, the query has almost certainly not overfit to the data, and the algorithm simply returns the empirical mean of the query on $S$ plus additional random noise. We show that the addition of noise ensures that the algorithm, as a function from the data sample $S$ to an answer, satisfies differential privacy. This can be leveraged to show that any query which depends on a limited number of previous queries will have an empirical mean on $S$ that is close to the true expectation. This ensures that \textsc{ValidationRound} can accurately answer a large number of queries, while allowing some (unknown) subset of the queries to be adaptive.

\textsc{ValidationRound} uses truncated Gaussian noise $\xi \sim \N(0,\sigma^2,[-\gamma,\gamma])$, i.e.~Gaussian noise $Z \sim \N(0,\sigma^2)$ conditioned on the event $\abs{Z} \leq \gamma$. Its density
$f_{\xi}(x) \propto \expo{-\frac{x^2}{2\sigma^2}} \mathbbm{1}_{\abs{x} \leq \gamma}$.

\begin{algorithm}
\caption{\textsc{ValidationRound}$(\tau,\beta,n,S,T)$}
\label{alg:everlasting_round}
\begin{algorithmic}[1]
\STATE Set $I(\tau,\beta,n) = \frac{\beta}{4}\expo{\frac{n\tau^2}{8}}$, $\sigma^2 = \frac{\tau^2}{32\ln\parens{8n^2/\beta}}$
\FOR{ each query $q_1,q_2,...$ }
\IF{ $\abs{\ee{S}{q_i} - \ee{T}{q_i}} \leq \frac{\tau}{2}$ \textbf{and} $i \leq I(\tau,\beta,n)$} \label{alg:overfitcheck}
\STATE Draw truncated Gaussian $\xi_i \sim \N(0,\sigma^2,[-\frac{\tau}{4},\frac{\tau}{4}])$
\STATE \textbf{Output:} $a_i = \ee{S}{q_i} + \xi_i$
\ELSE
\STATE \textbf{Halt} ($\eta = i$)
\ENDIF
\ENDFOR
\end{algorithmic}
\end{algorithm}

Here, $\eta$ is the index of the query that causes the algorithm to halt. If $\eta \leq I(\tau,\beta,n)$, the maximum allowed number of answers, we say that \textsc{ValidationRound} halted ``prematurely.'' The following three lemmas characterize the behavior of \textsc{ValidationRound}.

\begin{restatable}{lemma}{validityoneround}\label{lem:validity_oneround}
For any $\tau$, $\beta$, and $n$, for any sequence of querying rules (with arbitrary adaptivity) and any probability distribution $\D$, the answers provided by \textsc{ValidationRound}$(\tau,\beta,n,S,T)$ satisfy 
\[
\Prob{ \forall_{i<\eta} \abs{A_i - \underset{x\sim\D}{\mathbb{E}}\left[Q_i(x)\right]} \leq \tau } \geq 1 - \frac{\beta}{2},
\]
where the probability is taken over the randomness in the draw of datasets $S$ and $T$ from $\D^n$, the querying rules, and \textsc{ValidationRound}.
\end{restatable}

\begin{restatable}{lemma}{nonadaptivedoesnotendround}\label{lem:non-adaptive_does_not_end_round}
For any $\tau$, $\beta$, and $n$, any sequence of querying rules, and any non-adaptive user  $\set{u_j}_{j\in[M]}$ interacting with \textsc{ValidationRound}$(\tau,\beta,n,S,T)$,
$\Prob{\eta \leq I(\tau,\beta,n) \land \eta \in \set{u_j}_{j\in[M]}} \leq \beta.$
% where the probability is taken over the randomness in the draw of datasets $S$ and $T$ from $\D^n$, the querying rules, and \textsc{ValidationRound}.
\end{restatable}

\begin{restatable}{lemma}{manyadaptivequeriesinround}\label{lem:many_adaptive_queries_in_round}
For any $\tau$, $\beta$, and $n$, any sequence of querying rules, and any autonomous user  $\set{u_j}_{j\in[M]}$ interacting with \textsc{ValidationRound}$(\tau,\beta,n,S,T)$, if
$\sigma^2 = \frac{\tau^2}{32\ln\parens{8n^2/\beta}}$ and
$M \leq \frac{n^2\tau^4}{175760\ln^2\parens{8n^2/\beta}}$
then \\$\Prob{\eta \leq I(\tau,\beta,n) \land \eta \in \set{u_j}_{j\in[M]}} \leq \beta$.
% where the probability is taken over the randomness in the draw of datasets $S$ and $T$ from $\D^n$, the querying rules, and \textsc{ValidationRound}.
\end{restatable}
% In other words, Lemmas \ref{lem:non-adaptive_does_not_end_round} and \ref{lem:many_adaptive_queries_in_round} show that with high probability, a non-adaptive user will not cause \textsc{ValidationRound} to halt prematurely and an adaptive user will only do so if they have already made $\tilde{\Omega}(n^2)$ queries.
Lemma \ref{lem:validity_oneround} indicates that all returned answers are accurate with high probability, regardless of adaptivity. The proof involves showing that $\ee{T}{q_i}$ is close to $\E{q_i}$ for each query, so any query that is answered must be accurate since $\abs{\ee{S}{q_i} - \ee{T}{q_i}}$ and $\abs{\xi}$ are small. Lemma \ref{lem:non-adaptive_does_not_end_round} indicates that with high probability, non-adaptive queries never cause a premature halt, which is a simple application of Hoeffding's inequality. Finally, Lemma \ref{lem:many_adaptive_queries_in_round} shows that with high probability, an autonomous user who makes $\tilde{\mathcal{O}}(n^2)$ queries will not cause a premature halt. This requires showing that $\ee{S}{q_i}$ is close to $\E{q_i}$ despite the potential adaptivity.

% We sketch the proofs here, detailed proofs are contained in Appendix \ref{sec:section3proofs}. First, with high probability $\ee{T}{q_i}$ is close to $\E{q_i}$ for every query, regardless of adaptivity, which ensures validity since $\abs{\xi}$ is small (Lemma \ref{lem:validity_oneround}). Second, $\ee{S}{q_i}$ is also close to $\E{q_i}$ for all non-adaptive queries with high probability. This is a simple application of Hoeffding's inequality, and it ensures that non-adaptive queries will not cause \textsc{ValidationRound} to halt prematurely (Lemma \ref{lem:non-adaptive_does_not_end_round}). Finally, we show that for an autonomous user who makes $\tilde{\mathcal{O}}(n^2)$ queries, with high probability $\ee{S}{q_i}$ is close to $\E{q_i}$ despite the adaptivity. Consequently, the autonomous user will not cause \textsc{ValidationRound} to halt until it has made a quadratic number of queries (Lemma \ref{lem:many_adaptive_queries_in_round}). 

The proof of Lemma \ref{lem:many_adaptive_queries_in_round} uses existing results from adaptive data analysis together with a simple argument that noise truncation does not significantly affect the results. For reference, the results we cite are included in Appendix \ref{sec:dpfactoids}. While using Gaussian noise to answer queries is mentioned in other work, we are not aware of an explicit analysis, so we analyze the method here. To simplify parts of the derivation, we rely on the notion of concentrated differential privacy, which is particularly well suited for analysis of composition with Gaussian noise addition \cite{BunSteinke16}. Lemmas \ref{lem:validity_oneround}-\ref{lem:many_adaptive_queries_in_round} are proven in Appendix \ref{sec:section3proofs}.

\section{\textsc{EverlastingValidation} and pricing}\label{sec:everlastingvalidation}
\textsc{ValidationRound} uses a fixed number, $n$, of samples and with high probability returns accurate answers for at least $\expo{\Omega(n)}$ non-adaptive queries and $\tilde{\Omega}(n^2)$ adaptive queries. In order to handle infinitely many queries,
we chain together multiple instances of \textsc{ValidationRound}. We start with an initial dataset, answer queries using \textsc{ValidationRound} using that data until it halts. At this point, we buy more data and repeat. The used-up data can be released to the public as a ``training set,'' which can be used with no restriction without affecting any guarantees.

\begin{algorithm}[H]
\caption{\textsc{EverlastingValidation}$(\tau,\beta)$}
\label{alg:everlastingvalidation}
\begin{algorithmic}[1]
\STATE Require initial budget $\Gamma = 36\ln\parens{8/\beta}/\tau^2$
\STATE $N_0 = \frac{\Gamma}{2}$, $\beta_0 = \frac{\beta}{2}$, $t = 0$, $i = 0$
\STATE Buy datasets $S_0,T_0 \sim \D^{N_0}$
\LOOP
\STATE Pass $q_i$ to \textsc{ValidationRound}$(\tau,\beta_t,N_t,S_t,T_t)$
\IF{\textsc{ValidationRound} does not halt}
\STATE \textbf{Output:} $a_i$
\STATE Charge $\frac{96}{\tau^2}\cdot\frac{1}{i}$, move on to $i = i+1$
\ELSE
\STATE Charge $6N_t$ minus current capital
\STATE $N_{t+1} = 3N_t$, $\beta_{t+1} = \frac{1}{2}\beta_t$, $t = t+1$
\STATE Buy datasets $S_t,T_t \sim \D^{N_t}$
\STATE Restart loop with same $i$
\ENDIF
\ENDLOOP
\end{algorithmic}
\end{algorithm}

The key ingredient is a pricing system with which we can always afford new data when an instance of \textsc{ValidationRound} halts. Our method has two price types: a low price, which is charged for all queries and decreases like $1/i$; and a high price, which is charged for any query that causes an instance of \textsc{ValidationRound} to halt prematurely, which may grow with the size of the current dataset. \textsc{EverlastingValidation}$(\tau,\beta)$ guarantees the following:

\begin{theorem}[Validity]\label{thm:validity}
For any sequence of querying rules (with arbitrary adaptivity), \textsc{EverlastingValidation} will provide answers such that
\[
\Prob{ \forall_{i\in\bbN} \abs{A_i - \underset{x\sim\D}{\mathbb{E}}\left[Q_i(x)\right]} \leq \tau } \geq 1 - \frac{\beta}{2}
\]
\end{theorem}
\begin{proof}
Consider the sequence of query rules that are answered by the $t^{\textrm{th}}$ instantiation of the \textsc{ValidationRound} mechanism. By Lemma \ref{lem:validity_oneround}, for any sequence of querying rules, with probability $1-\frac{\beta_t}{2}$, all of the answers during round $t$ are answered accurately. By a union bound over all rounds, all answers in all rounds are accurate with probability at least $1-\sum_{t=0}^\infty \beta_t/2 = 1 - \beta/2$.
\end{proof}

\begin{theorem}[Sustainability]\label{thm:sustainability}
For any sequence of queries, the revenue collected can pay for all samples ever needed by \textsc{EverlastingValidation}, excluding the initial budget of $36\ln\parens{8/\beta}/\tau^2$.
\end{theorem}
\begin{proof}
When \textsc{ValidationRound} halts, we charge exactly enough for the next $S_t,T_t$ (line 10).
\end{proof}

\begin{restatable}{lemma}{revenuelemma}\label{lem:revenue}
If $N_0 \geq 18\ln(2)/\tau^2$ and $I(\tau,\beta_t,N_t) = (\beta_t/4)\expo{N_t\tau^2/8}$ queries are answered during round $t$, then at least $6N_t$ revenue is collected.
\end{restatable}
The proof of Lemma \ref{lem:revenue} involves a straightforward computation. We find an upper bound, $B_T$, on the number of queries made before round $T$ begins and then lower bound the revenue collected in round $T$ with $\sum_i\frac{96}{\tau^2(B_T + i)}$. We defer the details to Appendix \ref{sec:section4proofs}.

\begin{theorem}[Cost for non-adaptive users]\label{thm:cost_non-adaptive}
For any sequence of querying rules and any non-adaptive user indexed by $\set{u_j}_{j\in[M]}$, the cost to the user satisfies
\[
\Prob{ \sum\nolimits_{j \in [M]} P_{u_j} \leq \frac{96}{\tau^2}\parens{1 + \ln(M)} } \geq 1 - \beta .
\]
\end{theorem}
\begin{proof}
By Lemma \ref{lem:revenue}, if a round $t$ ends after $I(\tau,\beta_t,N_t)$ queries are answered, then the total revenue collected from queries in that round is at least $6N_t$, so the ``high price'' at the end of the round is $0$.
Consequently, a query $q_{u_j}$ from the non-adaptive user costs the low price $96/(\tau^2 u_j)$ unless it causes an instantiation of \textsc{ValidationRound} to halt prematurely. By Lemma \ref{lem:non-adaptive_does_not_end_round} and a union bound, this never occurs in any round with probability at least $1-\sum_{t=0}^\infty \beta_t = 1 - \beta$, and the cost to the user is
% \[
% \sum_{j=1}^M p_{u_j} = \sum_{j=1}^M \frac{96}{\tau^2 u_j} \leq \sum_{i=1}^M \frac{96}{\tau^2 i} \leq \frac{96}{\tau^2}\parens{1 + \ln(M)}. \qedhere
% \]
\[
\sum\nolimits_{j\in[M]} p_{u_j} = \sum\nolimits_{j\in[M]} \frac{96}{\tau^2 u_j} \leq \sum\nolimits_{j\in[M]} \frac{96}{\tau^2 i} \leq \frac{96}{\tau^2}\parens{1 + \ln(M)}. \qedhere
\]
\end{proof}

\begin{theorem}[Cost for adaptive users]\label{thm:cost_adaptive}
For any sequence of querying rules and any autonomous user indexed by $\set{u_j}_{j\in[M]}$, there is a fixed constant $c_0$ such that the cost to the user satisfies
\[
\Prob{ \sum\nolimits_{j \in [M]} P_{u_j} \leq c_0 \cdot \frac{\sqrt{M}\ln^2\parens{M/\beta}}{\tau^2}}  \geq 1 - \beta.
\]
\end{theorem}
\begin{proof}
Ideally, none of the $M$ queries causes a premature halt, and the total cost is at most 
% Ideally, all $M$ queries will cost the low price $\frac{96}{\tau^2 u_j}$, totalling at most 
$\frac{96}{\tau^2}\parens{1 + \ln(M)}$, but the adaptive user may cause rounds to end prematurely and pay up to $6N_t$. However, by Lemma \ref{lem:many_adaptive_queries_in_round}, with probability $1-\beta_t$ if one of the adaptive user's queries causes a round $t$ to end prematurely, then the amount of data, $N_t$, and the number of the user's queries answered in that round, $M_t$, must satisfy
\begin{equation}\label{eq:lb_n_answered}
M_t \geq \frac{N_t^2\tau^4}{175760\ln^2\parens{8N_t^2/\beta_t}}.
\end{equation}
Given $M$, there is a largest $t$ for which this is possible since $N_t = 3^tN_0$ and $\beta_t = 2^{-t}\beta_0$. That is,
\[
\frac{3^{2t}N_0^2\tau^4}{175760\ln\parens{18^{t}\cdot 8N_0^2/\beta_0}} \leq M
\]
which implies $t_{\textrm{max}} \leq \frac{1}{2}\ln\parens{24\sqrt{M}\ln\parens{144N_0/\beta_0}}$.
Let $\mathcal{T}$ be the set of rounds in which the adaptive user pays the high $6N_t$ price, then with probability at least $1-\sum_{t\in\mathcal{T}}\beta_t \geq 1 - \beta$, inequality \eqref{eq:lb_n_answered} holds for all $t \in \mathcal{T}$.
In this case, the total cost to the adaptive user is no more than
% \begin{align*}
% \sum_{t \in \mathcal{T}} 6N_t
% &\leq 6\sum_{t \in \mathcal{T}} \frac{420\sqrt{M_t}\ln\parens{8M_t^2/\beta_t}}{\tau^2} \\
% &\leq t_{\textrm{max}}\frac{2520\sqrt{M}\ln\parens{8M^2/\beta_{t_{\textrm{max}}}}}{\tau^2} \\
% &\leq \frac{1890\sqrt{M}\ln^2\parens{16M^2/\beta}}{\tau^2}. \qedhere
% \end{align*}
\begin{equation*}
\sum_{t \in \mathcal{T}} 6N_t
% \leq 6\sum_{t \in \mathcal{T}} \frac{420\sqrt{M_t}\ln\parens{8M_t^2/\beta_t}}{\tau^2} \\
\leq t_{\textrm{max}}\frac{2520\sqrt{M}\ln\parens{8M^2/\beta_{t_{\textrm{max}}}}}{\tau^2} 
\leq \frac{1890\sqrt{M}\ln^2\parens{16M^2/\beta}}{\tau^2}. \qedhere
\end{equation*}
\end{proof}

\paragraph{Relationship to prior work on adaptive data analysis}
We handle adaptivity using ideas developed in recent work on adaptive data analysis.
% To handle adaptivity we rely on ideas developed in the context of the recent work on adaptive data analysis. 
In this line of work, all queries are typically assumed to be adaptively chosen and the overall number of queries known in advance. For completeness, we briefly describe several algorithms that have been developed in this context and compare them with our algorithm.
Dwork et al.~\cite{DworkFHPRR14:arxiv} analyze an algorithm that adds Laplace or
Gaussian noise to the empirical mean in order to answer $M$ adaptive
queries using $\tilde{O}(\sqrt{M})$ samples---a method that forms the
basis of \textsc{ValidationRound}.  However, adding untruncated Laplace
or Gaussian noise to exponentially many non-adaptive queries would likely
cause large errors when the variance is large enough to ensure that the sample mean is accurate.
We use truncated Gaussian noise instead and show that it does not substantially affect the analysis for autonomous queries.

\textsc{Thresholdout} \cite{dwork2015generalization} answers verification queries in which the user submits both a query and an estimate of the answer. The algorithm uses $n = \tilde O(\sqrt{M}\cdot\log{I})$ samples to answer $I$ queries of which at most $M$ estimates are far from correct. Similar to our use of the second dataset $T$, this algorithm can be used to detect overfitting and answer adaptive queries (this is the basis of the \textsc{EffectiveRounds} algorithm \cite{DworkFHPRR14:arxiv}).
However, in our application this algorithm would have sample complexity of $n = \tilde O(\sqrt{M}\cdot\log{I})$, for $M$ autonomous queries in $T$ total queries. Consequently, direct use of this mechanism would result in a pricing for non-adaptive users that depends on the number of queries by autonomous users. This is in contrast to $n = \tilde O(\sqrt{M} + \log{T})$ samples that suffice for $\textsc{ValidationRound}$, where the improvement relies on our definition of autonomy and truncation of the noise variables.

% In other related work, \textsc{Ladder} \cite{BlumH15} and \textsc{ShakyLadder}
% \cite{Hardt17arxiv} are mechanisms tailored to maintaining a
% ML competition leaderboard.  These algorithms reveal the answer to a user's query (submission of a model) only when the error of the model is sufficiently lower than the error of the best previous submission of the user. In particular, while the mechanism can handle an exponential number of
% arbitrarily adaptive submissions, each user will receive answers only to a relatively small number of queries. Our setting is more suitable for the case where we want to validate the errors of all the submissions or for scientific discovery where there is more then one discovery to be made.
% At the same time we remark that ideas in those mechanisms can be easily combined with our approach of making the validity guarantees everlasting by collecting additional data and penalizing only adaptive users.

\section{Optimality}\label{sec:optimality}
One might ask if it is possible to devise a mechanism with similar
properties but lower costs. We argue that the prices set by
\textsc{EverlastingValidation} are near optimal.
The total cost to a non-adaptive user who makes $M$ queries is
$O(\log M/\tau^2)$.  Even if we knew in advance that we would receive only $M$
non-adaptive queries, we would still need $\Omega(\log M/\tau^2)$ samples to answer all of them
accurately with high probability.  Thus, our price for non-adaptive
queries is optimal up to constant factors. 
\blockremove{Note that in the presence
of additional structural assumptions on the set of queries (such as low VC dimension)
a smaller number of samples would suffice. Such structural assumptions can also be used to argue about the number of samples needed for our algorithm.}

It is also known that answering a sequence of $M$ adaptively chosen
queries with accuracy $\tau$ requires $\tilde\Omega(\sqrt{M}/\tau)$
samples \cite{HardtU14,SteinkeU15}. Hence, the cost to a possibly adaptive autonomous 
user is nearly optimal in its dependence on $M$ (up to log factors). 
One natural concern is that our
guarantee in this case is only for the amortized (or total) cost, and
not on the cost of each individual query.  Indeed, although the
\emph{average} cost of adaptive queries decreases as $\tilde{O}(1/\sqrt{M})$, 
the \emph{maximal} cost of a single query might
increase as $\tilde{O}(\sqrt{M})$. A natural question is whether the maximum price can
be reduced, to spread the high price over more queries.

Finally, an individual who queries our mechanism with $M$ entirely non-adaptive queries
will only pay $\log M$ in the worst case; generally, they will benefit from the economies of scale associated with collecting more and more data. For instance, if there are $K$ users each making $M$ non-adaptive queries, then the total cost of all $KM$ queries will be $\log KM$ so the average cost to
each user is only $\log(KM)/K \ll \log M$.

\section{An Alternative Approach: \textsc{EverlastingTO}}
The \textsc{EverlastingValidation} mechanism provides cost guarantees that are, in certain ways, nearly optimal. The two main shortcomings are that (1) the price is guaranteed only for non-adaptive or autonomous users--not arbitrary adaptive ones and (2) the cost of an individual adaptive query cannot be upper bounded. One might also ask if inventing \textsc{ValidationRound} was necessary in the first place. Another mechanism, \textsc{Thresholdout} \cite{dwork2015generalization}, is already well-suited to the setting of mixed adaptive and non-adaptive queries and it gives accuracy guarantees for quadratically many arbitrary adaptive queries or exponentially many non-adaptive queries. Perhaps using \textsc{Thresholdout} instead would be better? We will now describe an alternative mechanism, \textsc{EverlastingTO}, which allows us to provide price guarantees for individual queries, including arbitrarily adaptive ones, but with an exponential increase in the cost for both non-adaptive and adaptive queries.

The \textsc{EverlastingTO} mechanism is very similar to \textsc{EverlastingValidation}, except it uses \textsc{Thresholdout}  in the place of \textsc{ValidationRound}. In each round, the algorithm determines an overfitting budget, $B_t$, and a maximum number of queries, $M_t$, as a function of the tradeoff parameter $p$. It then answers queries using \textsc{Thresholdout}, charging a high price $2N_{t+1}/B_t$ for queries that fail the overfitting check, and charging a low price $2N_{t+1}/M_t$ for all of the other queries. Once \textsc{Thresholdout} cannot answer more queries, the mechanism buys more data, reinitializes \textsc{Thresholdout}, and continues as before.
\begin{algorithm}
\caption{\textsc{EverlastingTO}$(\tau,\beta,p)$}
\label{alg:everlastingTO}
\begin{algorithmic}[1]
\STATE Require sufficiently initial budget $n=n(\tau,\beta,p)$
\STATE $\forall t$ set $N_t = ne^t$, $\beta_t = \frac{(e-1)\beta}{e}e^{-t}$, $B_t = \tilde{\Theta}\parens{\frac{\tau^4N_t^{2-2p}}{\ln 1/\beta_t}}$, $M_t = \frac{\beta_t}{4}\expo{2N_t^p}$
\FOR{ $t=0,1,\dots$ }
\STATE Purchase datasets $S_t,T_t \sim \mathcal{D}^{N_t}$ and initialize \textsc{Thresholdout}$(S_t,T_t,B_t,\beta_t)$.
\WHILE{ \textsc{Thresholdout}$(S_t,T_t,B_t,\beta_t)$ has not halted }
\STATE Accept query $q$
\STATE $(a,o) = \textsc{Thresholdout}(S_t,T_t,B_t,\beta_t)(q)$
\STATE \textbf{Output}: $a$
\IF{ $o = \perp$ }
\STATE \textbf{Charge}: $\frac{2N_{t+1}}{M_t}$
\ELSE
\STATE \textbf{Charge}: $\frac{2N_{t+1}}{B_t}$
\ENDIF
\ENDWHILE
\ENDFOR
\end{algorithmic}
\end{algorithm}

We analyze \textsc{EverlastingTO} in Appendix \ref{sec:everlastingTO}. Theorems \ref{thm:TOvalidity}-\ref{thm:TOacost} closely parallel the guarantees of \textsc{EverlastingValidation} and establish the following for any $\tau,\beta \in (0,1)$ and any $p \in (0,\frac{2}{3})$: \textbf{Validity}: with high probability, for any sequence of querying rules, all answers provided by \textsc{EverlastingTO} are $\tau$-accurate. \textbf{Sustainability}: \textsc{EverlastingTO} charges high enough prices to be able to afford new samples as needed, excluding the initial budget. \textbf{Cost}: with high probability, any $M$ non-adaptive queries and any $B$ adaptive queries cost at most $O\parens{\ln^{1/p}(M) + B^{\frac{1}{2-3p}}}$ (ignoring the dependence on $\tau,\beta$). 

Unlike \textsc{EverlastingValidation}, which prioritized charging as little as possible for non-adaptive queries, \textsc{EverlastingTO} increases the $O(\log M)$ cost to $O(\polylog M)$ in order to bound the price of arbitrary adaptive queries. The parameter $p$ allows the database manager to control the tradeoff; for $p$ near zero, the cost of $B$ adaptive queries is roughly the optimal $O(\sqrt{B})$, but non-adaptive queries are extremely expensive. On the other side, for $p$ near $2/3$, the cost of adaptive queries becomes very high, but the cost of non-adaptive queries is relatively small, although it does not approach optimality.

Further details of the mechanism are contained in Appendix \ref{sec:everlastingTO}. We also provide a tighter analysis of the \textsc{Thresholdout} algorithm which guarantees accurate answers using a substantially smaller amount of data in Appendix \ref{sec:betterTO}. This analysis allows us to reduce the exponent in \textsc{EverlastingTO}'s cost guarantee for non-adaptive queries.

\section{Potential applications} \label{Sec:app}
In the ML challenge scenario, validation results are often displayed on a scoreboard.
Although it is often assumed that scoreboards cannot be used for extensive adaptation, 
it appears that such adaptations have
played roles in determining the outcome of various well known competitions,
including the Netflix challenge, where the final test set performance
was significantly worse than performance on the leaderboard data set. % and the Baidu ImageNet success \cite{wu2015deep}.
\textsc{EverlastingValidation} would guarantee that test errors returned
by the validation database are accurate, regardless of
adaptation, collusion, the number of queries made by each user, or other
intentional or unintentional dependencies.  We do charge a
price per-validation, but as long as users are non-adaptive, the price
is very small.  Adaptive users, on the other hand, pay what is required in order to ensure
validity (which could be a lot).  Nevertheless, even if a wealthy user could afford paying the
higher cost of adaptive queries, she would still not be able to
``cheat'' and overfit the scoreboard set, and a poor user could still
afford the quickly diminishing costs of validating non-adaptive queries.
% Furthermore, wealthy adaptive users can provably learn more about $\D$ by buying their own data
% than from posing adaptive queries to the scoreboard set.

Another feature of our mechanism is that
once a round $t$ is over, we can safely release the datasets $S_t$ and
$T_t$ to the public as unrestricted
training data. This way, poor analysts also benefit from adaptive
queries made by others, as all data is eventually released, and at any
given time, a substantial fraction of all the data ever collected is public.
Also, the ratio of public data to validation data can easily be adjusted
by slightly amending the pricing.

In the context of scientific discovery, one use case is very similar to the ML
competition. Scientists can search for
interesting phenomena using unprotected data, and then re-evaluate
``interesting'' discoveries with the database mechanism in order to get an
accurate and almost-unbiased estimate of the true value. This could be useful, for example,
in building prediction models for scientific phenomena
such as genetic risk of disease, which often involve complex modeling
\cite{chatterjee2016developing}.

However, most scientific research is done in the context of hypothesis testing rather than estimation. Declarations of discoveries like the Higgs boson \cite{higgs} and genetic associations of disease \cite{craddock2010genome} are based on performing a potentially large number of hypothesis tests and identifying statistically significant discoveries while controlling for multiplicity. Because of the complexity of the discovery process, it is often quite difficult to properly control for all potential tests, causing many difficulties, the most well known of which is the problem of publication bias (cf.~``Why Most Published Research Findings are False'' \cite{ioannidis2005most}). An alternative, approach that has gained popularity in recent years, is requiring replication of any declared discoveries on new and independent data \cite{baker20161}. Because the new data is used only for replication, it is much easier to control multiplicity and false discovery concerns.

Our everlasting database can be useful in both the discovery and replication phases. We now briefly explain how its validity guarantees can be used for multiplicity control in testing.
Assume we have a collection of hypothesis tests on functionals of $\D$ with null hypotheses:
$
H_{0i} : \E{q_i} = e_{0i}.
$
We employ our scheme to obtain estimates $A_i$ of $\E{q_i}$. Setting $\alpha = \beta/2$, Theorem (\ref{thm:validity}) guarantees:
$
\sum_i \mathbb{P}_{H_{0i}} \left[ \max_i  |A_i - e_{0i}| > \tau \right] \leq \alpha,
$
meaning that for any combination of true nulls, the rejection policy {\em reject if $|A_i - e_{0i}| > \tau$} makes no false rejections with probability at least $1-\alpha$, thus controlling the family-wise error rate (FWER) at level $\alpha$. This is easily used in the replication phase, where an entire community (say, type-I diabetes researchers) could share a single replication server using the everlasting database scheme in order to to guarantee validity. It could also be used in the discovery phase for analyses that can be described through a set of measurements and tests of the form above.

\section{Conclusion and extensions}
Our primary contribution is in designing a database mechanism that brings together two important properties that have not been previously combined: everlasting validity and robustness to adaptivity. Furthermore, we do so in an asymptotically efficient manner that guarantees that non-adaptive queries are inexpensive with high probability, and that the potentially high cost of handling adaptivity only falls upon truly adaptive users. Currently, there are large constants in the cost guarantees, but these are pessimistic and can likely be reduced with a tighter analysis and more refined pricing scheme. We believe that with some improvements, our scheme can form the basis of practical implementations for use in ML competitions and scientific discovery. Also, our cost guarantees themselves are worst-case and only guarantee a low price to entirely non-adaptive users. It would be useful to investigate experimentally how much users would actually end up being charged under ``typical use,'' especially users who are only ``slightly adaptive.'' However, there is no established framework for understanding what would constitute ``typical'' or ``slightly adaptive'' usage of a statistical query answering mechanism, so more work is needed before such experiments would be insightful.

Our mechanism can be improved in several ways. It only provides answers at a fixed, additive $\tau$, and only answers statistical queries, however these issues have been already addressed in the adaptive data analysis literature. E.g.~arbitrary low-sensitivity queries can be handled without any modification to the algorithm, and arbitrary real-valued queries can be answered with the error proportional to their standard deviation (instead of $1/\sqrt{n}$ as in our analysis) \cite{FeldmanS17}. These approaches can be combined with our algorithms but we restrict our attention to the basic case since our focus is different.

% It is also possible to spread the costs to adaptive users better at the expense of a somewhat more involved mechanism.
% We briefly outline how this can be done and leave a more formal treatment to future work. We add a third dataset (of
% the same size) and use it to obtain an estimate of the answer to the
% query $a'_i$, which we use as the estimated answer in the 
%  \textsc{Thresholdout} algorithm \cite{dwork2015generalization}
% with truncated Laplace noise instead of the standard Laplace. 
% The parameters of \textsc{Thresholdout} should allow at
% least $M$ failed verification steps out of the total $2M$ queries.
% Each time \textsc{Thresholdout} fails the verification step
% (which, with high probability, will only happen for adaptive
% autonomous users) we charge the user
% $O(\log(M)/\sqrt{M})$. With high probability,
% \textsc{ValidationRound} fails only after $M$ verification steps
% of \textsc{Thresholdout} have failed, so the total revenue
% is sufficient to continue.

Finally, one potentially objectionable element of our approach is that it discards samples at the end of each round (although these samples are not wasted since they become part of the public dataset). An alternative approach is to add the new samples to the dataset as they can be purchased. While this might be a more practical approach, existing analysis techniques that are based on differential privacy do not appear to suffice for dealing with such mechanisms. Developing more flexible analysis techniques for this purpose is another natural direction for future work.

\paragraph{Acknowledgements} 
BW is supported the NSF Graduate Research Fellowship under award 1754881. NS is supported by NSF-BSF grant number 1718970 and by NSF BIGDATA grant number 1546500.

\bibliography{everlastbib,everlast-vitaly}
\bibliographystyle{plainnat}

\newpage
\appendix

\section{Proofs from Section \ref{sec:validationround}}\label{sec:section3proofs}

\begin{restatable}{lemma}{Tclosetopop}\label{lem:T_is_close_to_pop}
For any $\tau$, $\beta$, $n$, and any sequence of querying rules (with arbitrary adaptivity) interacting with \textsc{ValidationRound}$(\tau,\beta,n,S,T)$
\[
\Prob{ \forall_{i < \eta}\ \abs{\ee{T}{Q_i} - \underset{x\sim\D}{\mathbb{E}}\left[Q_i(x)\right]} \leq \frac{\tau}{4} } \geq 1 - \frac{\beta}{2}.
\]
% where the probability is taken over the randomness in the draw of datasets $S$ and $T$ from $\D^n$, the querying rules, and \textsc{ValidationRound}.
\end{restatable}
% \Tclosetopop*
\begin{proof}
Consider any sequence of querying rules (with arbitrary adaptivity). The interaction between the query rules and \textsc{ValidationRound}$(\tau,\beta,n,S,T)$ together determines a joint distribution over statistical queries, answers, and prices $(Q_1,A_1,P_1),...,(Q_{\eta-1},A_{\eta-1},P_{\eta-1})$.

Consider also the interaction of the same sequence of querying rules with an alternative algorithm, which always returns $\ee{S}{q_i} + \xi_i$ (i.e.~it ignores the if-statement in \textsc{ValidationRound}). This generates an infinite sequence of queries, answers, and prices $(Q_1',A_1',P_1'),(Q_2',A_2',P_2'),...$. Now, we retroactively check the condition in the if-statement for each of the queries to calculate what $\eta$ should be, and take the length $\eta-1$ prefix of the $(Q_i',A_i',P_i')$. This sequence has exactly the same distribution as the sequence generated by \textsc{ValidationRound}, and each $Q_i'$ was chosen independently of $T$ by construction.
Since $Q_i'\sim Q_i$ has outputs bounded in $[0,1]$, we can apply Hoeffding's inequality:
\begin{align*}
\mathbb{P}\left[ \abs{\ee{T}{Q_i} - \underset{x\sim\D}{\mathbb{E}}\left[Q_i(x)\right]} > \frac{\tau}{4} \right]
% &= \Prob{\abs{\ee{T}{Q_i'} - \underset{x\sim\D}{\mathbb{E}}\left[Q_i'(x)\right]} > \frac{\tau}{4}} \\
&\leq 2\expo{-\frac{n\tau^2}{8}}.
\end{align*}
At most $I(\tau,\beta,n) = \frac{\beta}{4}\expo{\frac{n\tau^2}{8}}$ queries are answered by the mechanism, so a union bound completes the proof.
\end{proof}

\validityoneround*
\begin{proof}
A query is not answered unless $\abs{\ee{S}{q_i} - \ee{T}{q_i}} \leq \frac{\tau}{2}$, so $\forall i < \eta$
\begin{align*}
\abs{a_i - \E{q_i}} &\leq \abs{\xi_i} + \abs{\ee{S}{q_i} - \ee{T}{q_i}} + \abs{\ee{T}{q_i} - \E{q_i}} 
\leq \tau/4 + \tau/2 + \abs{\ee{T}{q_i} - \E{q_i}}.
\end{align*}
By Lemma \ref{lem:T_is_close_to_pop}, with probabilty $1-\frac{\beta}{2}$ the final term is at most $\tau/4$ simultaneously for all $i < \eta$.
\end{proof}

\nonadaptivedoesnotendround*
\begin{proof}
Since the non-adaptive user's querying rules ignore all of the history, they are each chosen independently of $S$. By Hoeffding's inequality
\[
\Prob{\abs{\ee{S}{Q_{u_j}} - \ExD{Q_{u_j}(x)}} > \frac{\tau}{4}} \leq 2\expo{-\frac{n\tau^2}{8}}
\]
and similarly for $T$. If both $\eta \leq I(\tau,\beta,n)$ and $\eta = u_j$, then the algorithm halted upon receiving query $q_{u_j}$ because its empirical means on $S$ and $T$ were too dissimilar and \emph{not} because it had already answered its maximum allotment of queries. Therefore,
\begin{align*}
\Prob{\eta \leq I(\tau,\beta,n) \land \eta = u_j} 
= \Prob{\abs{\ee{S}{Q_{u_j}} - \ee{T}{Q_{u_j}}} > \frac{\tau}{2}}
\leq 4\expo{-\frac{n\tau^2}{8}}.
\end{align*}
At most $I(\tau,\beta,n) = \frac{\beta}{4}\expo{\frac{n\tau^2}{8}}$ queries are answered by the mechanism, so a union bound completes the proof.
\end{proof}

\begin{restatable}{lemma}{adaptiveSclosetopop}\label{lem:adaptive_S_close_to_pop}
For any $\tau$, $\beta$, $n$, any sequence of query rules, and any possibly adaptive autonomous user $\set{u_j}_{j\in[M]}$, if
$\sigma^2 = \frac{\tau^2}{32\ln\parens{8n^2/\beta}}$ and
$M \leq \frac{n^2\tau^4}{175760\ln^2\parens{8n^2/\beta}}$
then
\[
\Prob{\forall_{j\in[M]}\ \abs{\ee{S}{Q_{u_j}} - \ExD{Q_{u_j}(x)}} \leq \frac{\tau}{4}} \geq 1 - \frac{\beta}{2}.
\]
% where the probability is taken over the randomness in the draw of datasets $S$ and $T$ from $\D^n$, the querying rules, and \textsc{ValidationRound}.
\end{restatable}
% \adaptiveSclosetopop*
\begin{proof}
Consider a slightly modified version of \textsc{ValidationRound}, where Gaussian noise $z_i \sim \N(0,\sigma^2)$ is added instead of truncated Gaussian noise $\xi_i$. Until this modified algorithm halts, all of the answers it provides are released according to the Gaussian mechanism on $S$, which satisfies $\frac{1}{2n^2\sigma^2}$-zCDP by
Proposition 1.6 in \cite{BunSteinke16}.
% Proposition \ref{prop:gaussian_mechanism}.
We can view $Q_{u_j} = R_{u_j}((q_{u_1},a_{u_1},p_{u_1})...,(q_{u_{j-1}},a_{u_{j-1}},p_{u_{j-1}}))$ as an (at most) $M$-fold composition of $\frac{1}{2n^2\sigma^2}$-zCDP mechanisms, which satisfies $\frac{M}{2n^2\sigma^2}$-zCDP by
Lemma 1.7 in \cite{BunSteinke16}.
% Proposition \ref{prop:zcdp_composition}.
Finally,
Proposition 1.3 in \cite{BunSteinke16}
% Proposition \ref{prop:zcdp_to_dp}
shows us how to convert this concentrated differential privacy guarantee to a regular differential privacy guarantee. In particular, $q_{u_j}$ is generated under
\[
\parens{\frac{M}{2n^2\sigma^2} + 2\sqrt{\frac{M}{2n^2\sigma^2}\ln\parens{1/\delta}},\ \delta}\textrm{-DP }\quad\forall \delta > 0.
\]
Specifically, when $\sigma^2$, $\delta$ and $M$ satisfy:
\begin{align*}
\sigma^2 &= \frac{\tau^2}{32\ln\parens{8n^2/\beta}} \\
\delta &= \frac{\beta}{8n^2} = \frac{\beta}{\frac{n^2\tau}{13\ln\parens{104/\tau}}}\cdot\frac{\tau}{104\ln\parens{104/\tau}} \\
M &\leq \frac{n^2\tau^4}{175760\ln^2\parens{8n^2/\beta}}.
\end{align*}
then $q_{,i_j}$ is generated by a $\parens{\frac{\tau}{52}, \delta}$-differentially private mechanism. Therefore, by
Theorem 8 in \cite{dwork2015generalization} (cf.~\cite{NissimS15,BassilyNSSSU16})
% Proposition \ref{prop:dp_generalizes}
\[
\Prob{\abs{\ee{S}{q_{u_j}} - \E{q_{u_j}}} > \frac{\tau}{4}} \leq \frac{\beta}{\frac{n^2\tau}{13\ln\parens{104/\tau}}} \ll \frac{\beta}{4M}.
\]

Furthermore, for $z_i \sim \N\parens{0,\sigma^2}$
% \[
$\Prob{\abs{z_i} \geq \tau/4} \leq \beta/(4n^2) \leq \beta/(4M)$.
% \]
Therefore, the total variation distance between $\xi_{u_j} \sim \N\parens{0,\sigma^2,[-\tau/4,\tau/4]}$ and $z_{u_j} \sim \N(0,\sigma^2)$ is $\Delta(\xi_{u_j},z_{u_j}) = \Prob{z_{u_j} \not\in [-\tau/4,\tau/4]} \leq \frac{\beta}{4M}$. Consider two random vectors $Z$ and $\xi$, the first of which has independent $\N(0,\sigma^2)$ distributed coordinates, and the second of which has coordinates $\xi_{u_j} \sim \N\parens{0,\sigma^2,[-\tau/4,\tau/4]}$ for $j\in[M]$ and $\xi_i = Z_i$ for all of the $i \not\in \set{u_j}$. The total variation distance between these vectors is then at most
 $\Delta(\xi,Z) \leq M\Delta(\xi_{u_j},z_{u_j}) \leq \frac{\beta}{4}$.

Now, for the given sequence of querying rules, $S$, and $T$, view \textsc{ValidationRound} as a function of the random noise which is added into the answers. Then $\Delta(\textrm{\textsc{ValidationRound}}(\xi), \textrm{\textsc{ValidationRound}}(Z)) \leq \Delta(\xi,Z) \leq \frac{\beta}{4}$ too.
Above, we showed that with probability $1-\beta/4$ the user's interaction with \textsc{ValidationRound}$(Z)$ has the property that
\[
\Prob{\exists_{j\in[M]}\ \abs{\ee{S}{q_{u_j}} - \E{q_{u_j}}} > \frac{\tau}{4}} \leq \frac{\beta}{4}.
\]
So their interaction with \textsc{ValidationRound}$(\xi)$ satisfies
\[
\Prob{\exists_{j\in[M]}\ \abs{\ee{S}{q_{u_j}} - \E{q_{u_j}}} > \frac{\tau}{4}} \leq \frac{\beta}{2}.
\]
Since this statement only depends on the indices of $\xi$ in $\set{u_j}_{j\in[M]}$, we can replace all of the remaining indices with truncated Gaussians and maintain this property, which recovers \textsc{ValidationRound}.
\end{proof}

\manyadaptivequeriesinround*
\begin{proof}[Proof of Lemma \ref{lem:many_adaptive_queries_in_round}]
Consider a query $q_{u_j}$ made by the autonomous user.
Lemma \ref{lem:T_is_close_to_pop} guarantees that
\[
\Prob{\forall_{j\in[M]}\ \abs{\ee{T}{q_{u_j}} - \E{q_{u_j}}} \leq \frac{\tau}{4}} \geq 1 - \frac{\beta}{2}.
\]
By Lemma \ref{lem:adaptive_S_close_to_pop}, with the hypothesized $\sigma^2$ and $M$
\[
\Prob{\forall_{j\in[M]}\ \abs{\ee{S}{q_{u_j}} - \E{q_{u_j}}} \leq \frac{\tau}{4}} \geq 1 - \frac{\beta}{2}.
\]
If both $\eta \leq I(\tau,\beta,n)$ and $\eta \in \set{u_j}_{j\in[M]}$, then the algorithm halted upon receiving a query $q_{u_j}$ because its empirical means on $S$ and $T$ were too dissimilar and \emph{not} because it had already answered its maximum allotment of queries:
\begin{equation*}
\Prob{\eta \leq I(\tau,\beta,n) \land \eta \in \set{u_j}_{j\in[M]}}
= \Prob{\exists_{j\in[M]}\ \abs{\ee{S}{q_{u_j}} - \ee{T}{q_{u_j}}} > \frac{\tau}{2}}
\leq \beta .\qedhere
\end{equation*}
\end{proof}

\section{Proofs of Lemma \ref{lem:revenue}}\label{sec:section4proofs}
\revenuelemma*
\begin{proof}
The revenue collected in round $t$ via the low price $\frac{96}{\tau^2 i}$ depends on how many queries are answered both in and before round $t$. The maximum number of queries answered in a round is $I_t = I(\tau,\beta_t,N_t) = (\beta_t/4)\expo{N_t\tau^2/8}$ (this is enforced by \textsc{ValidationRound}).
Let $B_T$ be the total number of queries made before the beginning of round $T$, then
\begin{align*}
B_T \leq \sum_{t=0}^{T-1} I_t
&= \sum_{t=0}^{T-1} \frac{\beta_t}{4}\expo{\frac{N_t\tau^2}{8}} \\
% &= \frac{\beta_0}{4}\sum_{t=0}^{T-1} \expo{\frac{\tau^2}{8}3^tN_0 - t\ln2} \\
&\leq \frac{\beta_0}{4} \expo{\sum_{t=0}^{T-1} \frac{\tau^2}{8}3^tN_0 - t\ln2} \\
% &= \frac{\beta_0}{4} \expo{-\frac{T(T-1)}{2}\ln2 + \frac{3^T-1}{2}\frac{N_0\tau^2}{8}} \\
% &= \frac{\beta_02^{-T}}{4}  \expo{\frac{-T^2 + 3T - \frac{N_0\tau^2}{8\ln2}}{2}\ln2 + \frac{N_T\tau^2}{16}} \\
&\leq (\beta_T/4)\expo{N_T\tau^2/16} .
\end{align*}
The first inequality holds because every exponent in the sum is at least $\ln(2)$ by our choice of $N_0$ and for any $x,y\geq \ln2$, $e^{x+y} \geq 2\max\parens{e^x,e^y} \geq e^x + e^y$.
The second inequality holds since $N_0 > \frac{18\ln2}{\tau^2}$ implies $-T^2 + 3T - N_0\tau^2/(8\ln2) \leq 0$.
So, if $I_T$ queries are answered during round $T$, the revenue collected is at least
\begin{align*}
\sum_{i=1}^{I_T} \frac{96}{\tau^2(B_T + i)}
% &= \frac{96}{\tau^2}\sum_{i=1}^{B_T+I_T} \frac{1}{i} - \sum_{i=1}^{B_T} \frac{1}{i} \\
&\geq \frac{96}{\tau^2}\parens{\ln\parens{B_T + I_T} - \ln\parens{B_T}} \\
% &= \frac{96}{\tau^2}\ln\parens{1 + \frac{I_T}{B_T}} \\
&\geq \frac{96}{\tau^2}\ln\parens{1 + \frac{(\beta_T/4)\expo{N_T\tau^2/8}}{(\beta_T/4)\expo{N_T\tau^2/16}}} \\
% &\geq \frac{96}{\tau^2}\ln\parens{1 + \frac{\frac{\beta_T}{4}\expo{\frac{N_T\tau^2}{8}}}{\frac{\beta_T}{4}\expo{\frac{N_T\tau^2}{16}}}} \\
% &= \frac{96}{\tau^2}\ln\parens{1 + \expo{\frac{N_T\tau^2}{16}}} \\
&\geq 6N_T\qedhere
\end{align*}
\end{proof}

\section{Tighter \textsc{Thresholdout} Analysis}\label{sec:betterTO}
In this section, we provide a tighter analysis of the \textsc{Thresholdout} algorithm \cite{dwork2015generalization}. In particular, previous analysis showed a sample complexity for answering $m$ queries with an overfitting budget of $B$ of $\tilde O(\sqrt{B}\ln^{1.5}m)$ whereas we prove a bound like $\tilde O(\sqrt{B}\ln m)$. The improvement has important consequences for our application of \textsc{Thresholdout} to the everlasting database setting. We make the improvement by applying the ``monitor technique'' of \citet{BassilyNSSSU16}.
\begin{algorithm}
\caption{\textsc{Thresholdout}$(S,T,\tau,\beta,\zeta,B,\sigma)$}
\label{alg:thresholdout}
\begin{algorithmic}[1]
\STATE Sample $\rho \sim \laplace{2\sigma}$
\FOR{ each query $q$ }
\IF{ $B < 1$ }
\STATE \textbf{HALT}
\ELSE
\STATE Sample $\lambda \sim \laplace{4\sigma}$
\IF{ $\abs{\ee{S}{q} - \ee{T}{q}} > \zeta + \rho + \lambda$}
\STATE Sample $\xi \sim \laplace{\sigma}$, $\rho \sim \laplace{2\sigma}$
\STATE $B \gets B - 1$
\STATE \textbf{Output}: $(\ee{T}{q} + \xi,\top)$
\ELSE
\STATE \textbf{Output}: $(\ee{S}{q},\perp)$
\ENDIF
\ENDIF
\ENDFOR
\end{algorithmic}
\end{algorithm}

\begin{lemma}[Lemma 23 \cite{dwork2015generalization}]\label{lem:TODP}
\textsc{Thresholdout} satisfies $\parens{\frac{2B}{\sigma n}, 0}$-differential privacy and also $\parens{\frac{\sqrt{32B\ln\parens{2/\delta}}}{\sigma n}, \delta}$-differential privacy for any $\delta > 0$.
\end{lemma}

\begin{lemma}[Corollary 7 \cite{dwork2015generalization}]\label{lem:pureDPgeneralization}
Let $\mathcal{A}$ be an algorithm that outputs a statistical query $q$. Let $S$ be a random dataset chosen according to distribution $\mathcal{D}^n$ and let $q = \mathcal{A}(S)$. If $\mathcal{A}$ is $\epsilon$-differentially private then 
\[
\Prob{\abs{\ee{S}{q} - \E{q}} \geq \epsilon} \leq 6\expo{-n\epsilon^2}
\]
\end{lemma}

\begin{lemma}[Theorem 8 \cite{dwork2015generalization}]\label{lem:DPgeneralization}
Let $\mathcal{A}$ be an $(\epsilon, \delta)$-differentially private algorithm that outputs a statistical query. For dataset $S$ drawn from $\mathcal{D}^n$, we let $q = \mathcal{A}(S)$. Then for $n \geq \frac{2\ln(8/\delta)}{\epsilon^2}$,
\[
\mathbb{P}\bracks{\abs{\ee{S}{q} - \E{q}} > 13\epsilon} \leq \frac{2\delta}{\epsilon}\ln\parens{\frac{2}{\epsilon}}
\]
\end{lemma}

\begin{theorem}[cf. Theorem 25 \cite{dwork2015generalization}]\label{thm:TO}
Let $\beta, \tau > 0$ and $m \geq B > 0$. Set $\zeta = \frac{3\tau}{4}$ and $\sigma = \frac{\tau}{48\ln\parens{4m/\beta}}$. Let $S,T$ denote datasets of size $n$ drawn i.i.d.~from a distribution $\mathcal{D}$. Consider an analyst that is given access to $S$ and adaptively chooses functions $q_1,\dots,q_m$ while interacting with \textsc{Thresholdout} which is given datasets $S,T$ and values $\sigma, B, \zeta$. For every $i \in [m]$ let $(a_i,o_i)$ denote the answer of \textsc{Thresholdout} on query $q_i$. 
Then whenever 
\[
n \geq \min\set{\mathcal{O}\parens{\frac{B\ln\parens{\frac{m}{\beta}}}{\tau^2}},\ \mathcal{O}\parens{\frac{\ln\parens{\frac{m}{\beta}}\sqrt{B\ln\parens{\frac{\ln\parens{1/\tau}}{\beta\tau}}}}{\tau^2}}}
\]
with probability at least $1-\beta$, for all $i$ before \textsc{Thresholdout} halts $\abs{a_i - \E{q_i}} \leq \tau$ and $o_i = \top \implies q_i$ is an adaptive query.
\end{theorem}
\begin{proof}
Consider the following post-processing of the output of \textsc{Thresholdout}: look through the sequence of queries and answers $\parens{q_1,a_1},\dots,\parens{q_{\textrm{HALT}},a_{\textrm{HALT}}}$ and output $q^*,a^* = \argmax_{q,a} \abs{a - \E{q}}$. Since this procedure does not use the datasets $S,T$ and since \textsc{Thresholdout} computes the sequence of queries and answers in a differentially private manner, it means that $q^*,a^*$ are also released under differential privacy. So by Lemma \ref{lem:TODP}, $q^*$ is released simultaneously under 
\begin{equation}
\parens{\frac{2B}{\sigma n}, 0}\textrm{-differential privacy} 
\qquad\mathand\qquad
\parens{\frac{\sqrt{32B\ln\parens{2/\delta}}}{\sigma n}, \delta}\textrm{-differential privacy}
\end{equation}
With our choice of $\sigma$, in the case that $n \geq \frac{768B\ln\parens{\frac{4m}{\beta}}}{\tau^2}$ then, using the pure differential privacy guarantee we have
$\frac{2B}{\sigma n} \leq \frac{\tau}{8}$
so by Lemma \ref{lem:pureDPgeneralization}
\begin{equation}
\mathbb{P}\bracks{\abs{\ee{T}{q^*} - \E{q^*}} > \frac{\tau}{8}} \leq \frac{\beta}{4}
\end{equation}

Alternatively, in the case that 
\[
n \geq \max\set{\frac{9984\ln\parens{\frac{4m}{\beta}}\sqrt{32B\ln\parens{\frac{1664\ln\parens{\frac{208}{\tau}}}{\beta\tau}}}}{\tau^2},\ \frac{21632\ln\parens{\frac{6656\ln\parens{\frac{208}{\tau}}}{\beta\tau}}}{\tau^2}}
\]
then, choosing $\delta = \frac{\beta\tau}{832\ln\parens{\frac{208}{\tau}}}$, under the approximate differential privacy guarantee we have
\begin{equation}\label{eq:algisapproxdp}
\parens{\frac{\sqrt{32B\ln\parens{2/\delta}}}{\sigma n}, \delta} 
\preceq 
\parens{\frac{\tau}{104}, \frac{\beta\tau}{832\ln\parens{\frac{208}{\tau}}}}
\end{equation}
so by Lemma \ref{lem:DPgeneralization}
\begin{equation}
\mathbb{P}\bracks{\abs{\ee{T}{q^*} - \E{q^*}} > \frac{\tau}{8}} \leq \frac{\beta}{4}
\end{equation}
Therefore, in either case $\mathbb{P}\bracks{\abs{\ee{T}{q^*} - \E{q^*}} > \frac{\tau}{8}} \leq \frac{\beta}{4}$.

Next, we note that the random variable $\lambda$ is sampled at most $m$ times, and the random variables $\rho$ and $\xi$ are sampled at most $B$ times. Consequently,
\begin{align}
\Prob{\exists i\ \abs{\lambda_i} > \frac{\tau}{12}} \leq m\cdot \Prob{\abs{\laplace{\frac{\tau}{12\ln\parens{4m/\beta}}}} > \frac{\tau}{12}} &\leq \frac{\beta}{4} \\
\Prob{\exists i\ \abs{\rho_i} > \frac{\tau}{24}} \leq B\cdot \Prob{\abs{\laplace{\frac{\tau}{24\ln\parens{4m/\beta}}}} > \frac{\tau}{24}} &\leq \frac{\beta}{4} \\
\Prob{\exists i\ \abs{\xi_i} > \frac{7\tau}{8}} \leq B\cdot \Prob{\abs{\laplace{\frac{\tau}{48\ln\parens{4m/\beta}}}} > \frac{7\tau}{8}} &\leq \frac{\beta}{8}
\end{align}

For the rest of the proof, we condition on the events $\abs{\ee{T}{q^*} - \E{q^*}} \leq \frac{\tau}{8}$ and $\forall i$ $\abs{\lambda_i} < \frac{\tau}{12}$, $\abs{\rho_i} < \frac{\tau}{24}$, and $\abs{\xi_i} < \frac{7\tau}{8}$. This event happens with probability $1-\frac{7\beta}{8}$.

Consider two alternatives: either $a^* = \ee{T}{q^*} + \xi^*$ or $a^* = \ee{S}{q^*}$. In the first case, 
\begin{align}
\abs{a^* - \E{q^*}}
\leq \abs{a^* - \ee{T}{q^*}} + \abs{\xi^*}
\leq \frac{\tau}{8} + \frac{7\tau}{8} = \tau
\end{align}
In the second case, we also have that $\abs{\ee{S}{q^*} - \ee{T}{q^*}} < \zeta + \rho^* + \lambda^*$, so
\begin{multline}
\abs{a^* - \E{q^*}}
\leq \abs{\ee{S}{q^*} - \ee{T}{q^*}} + \abs{\ee{T}{q^*} - \E{q^*}}
\leq 
 \zeta + \abs{\rho^*} + \abs{\lambda^*} + \frac{\tau}{8}
\leq \frac{3\tau}{4} + \frac{\tau}{24} + \frac{\tau}{12} + \frac{\tau}{8} = \tau
\end{multline}
Therefore, for all queries before \textsc{Thresholdout} halts, $\abs{a_i - \E{q_i}} \leq \tau$.

Next, observe that if $q$ is a non-adaptive query, then 
\begin{align}
\Prob{\abs{\ee{S}{q} - \E{q}} > \frac{\tau}{4}} = \Prob{\abs{\ee{T}{q} - \E{q}} > \frac{\tau}{4}} 
&\leq 2\expo{-\frac{\tau^2n}{8}} \leq 2\expo{50\ln\parens{\frac{\beta}{4m}}} \leq \frac{2\beta}{m\cdot4^{50}}
\end{align}
Therefore, with probability at least $1-\frac{\beta}{8}$, for all non-adaptive queries $\abs{\ee{S}{q} - \ee{T}{q}} \leq \frac{\tau}{2}$. Furthermore,
\begin{equation}
\zeta + \rho + \lambda \geq \frac{3\tau}{4} - \frac{\tau}{24} - \frac{\tau}{12} = \frac{5\tau}{8}
\end{equation}
Thus, for all non-adaptive queries $\abs{\ee{S}{q_i} - \ee{T}{q_i}} \leq \zeta + \rho_i + \lambda_i$, so $o_i = \perp$.
\end{proof}

\section{Guarantees of \textsc{EverlastingTO}}\label{sec:everlastingTO}
\begin{algorithm}
\caption{\textsc{EverlastingTO}$(\tau,\beta,p)$}
\begin{algorithmic}[1]
\STATE Require sufficiently large initial budget $n$ (see proof of Theorem \ref{thm:TOvalidity})
\STATE $\forall t$ set $N_t = ne^t$, $\beta_t = \frac{(e-1)\beta}{e}e^{-t}$, $B_t = \frac{\tau^4N_t^{2-2p}}{8\cdot9984^2\ln\parens{\frac{1664\ln\parens{\frac{208}{\tau}}}{\tau\beta_t}}}$, $M_t = \frac{\beta_t}{4}\expo{2N_t^p}$
\FOR{ $t=0,1,\dots$ }
\STATE Purchase datasets $S_t,T_t \sim \mathcal{D}^{N_t}$ and initialize \textsc{Thresholdout}$(S_t,T_t,B_t,\beta_t)$
\WHILE{ \textsc{Thresholdout}$(S_t,T_t,B_t,\beta_t)$ has not halted }
\STATE Accept query $q$
\STATE $(a,o) = \textsc{Thresholdout}(S_t,T_t,B_t,\beta_t)(q)$
\STATE \textbf{Output}: $a$
\IF{ $o = \perp$ }
\STATE \textbf{Charge}: $\frac{2N_{t+1}}{M_t}$
\ELSE
\STATE \textbf{Charge}: $\frac{2N_{t+1}}{B_t}$
\ENDIF
\ENDWHILE
\ENDFOR
\end{algorithmic}
\end{algorithm}

\begin{restatable}{theorem}{TOvalidity}[Validity]\label{thm:TOvalidity}
For any $\tau,\beta,p \in (0,1)$ and for a sufficiently large initial budget and for any sequence of queries, \textsc{EverlastingTO} returns answers such that
\[
\Prob{\exists i\ \abs{a_i - \E{q_i}} > \tau} < \beta
\]
\end{restatable}
\begin{proof}
In round $t$, the algorithm uses an instance of \textsc{Thresholdout} with $N_t$ samples for the datasets $S_t$ and $T_t$, so to answer $M_t$ total queries of which at most $B_t$ overfit we need both
\begin{align}
N_t = ne^t &\geq \frac{21632\ln\parens{\frac{6656\ln\parens{\frac{208}{\tau}}}{\tau\beta_t}}}{\tau^2}\label{eq:toguaranteelbnt1} \\
N_t = ne^t &\geq \frac{9984\ln\parens{\frac{4M_t}{\beta_t}}\sqrt{32B_t\ln\parens{\frac{1664\ln\parens{\frac{208}{\tau}}}{\tau\beta_t}}}}{\tau^2} \label{eq:toguaranteelbnt2}
\end{align}
in order to satisfy the hypotheses of Theorem \ref{thm:TO}.
Setting the constant $n$ such that
\begin{equation}\label{eq:n-lowerbound-1}
n \geq \frac{21632\parens{1 + \ln\parens{\frac{6656e\ln\parens{\frac{208}{\tau}}}{(e-1)\tau\beta}}}}{\tau^2}
\end{equation}
ensures that \eqref{eq:toguaranteelbnt1} holds. Furthermore, with our choice of
\begin{equation}
B_t = \frac{\tau^4N_t^{2-2p}}{8\cdot9984^2\ln\parens{\frac{1664\ln\parens{\frac{208}{\tau}}}{\tau\beta_t}}}
\end{equation}
the condition \eqref{eq:toguaranteelbnt2} allows us to answer
$M_t = \frac{\beta_t}{4}\expo{2N_t^p}$
total queries. 

We also need to ensure that $1 \leq B_t \leq M_t$ $\forall t$ in order to ensure that \textsc{Thresholdout} has sound parameters. To satisfy $1 \leq B_t$ requires the initial budget $n$ to be sufficiently large as $p \to 1$.
\begin{equation}
1\leq \frac{\tau^4\parens{ne^t}^{2-2p}}{8\cdot9984^2\ln\parens{\frac{1664\ln\parens{\frac{208}{\tau}}}{\tau\beta_t}}} \forall t
\iff
n \geq
\sup_{t\in\mathbb{N}} e^{-t} \parens{\frac{8\cdot9984^2\parens{t + \ln\parens{\frac{1664e\ln\parens{\frac{208}{\tau}}}{(e-1)\tau\beta}}}}{\tau^4}}^{\frac{1}{2-2p}}
% p\leq \inf_{t \in \mathbb{N}} 1 - \frac{\ln\parens{\frac{8\cdot9984^2}{\tau^4}\parens{t + \ln\parens{\frac{1664e\ln\parens{\frac{208}{\tau}}}{(e-1)\tau\beta}}}}}{2t + 2\ln n}
\end{equation}
By Lemma \ref{lem:sup-bound-for-bt-geq-1}, it thus suffices to choose
\begin{equation}\label{eq:n-lowerbound-2}
    n \geq \parens{\frac{8\cdot9984^2\ln\parens{\frac{1664e\ln\parens{\frac{208}{\tau}}}{(e-1)\tau\beta}}}{\tau^4}+\frac{4\cdot9984^2}{(1-p)\tau^4}}^\frac{1}{2-2p}
\end{equation}

At the same time, we need the initial budget to be large enough that $\forall t$ $B_t \leq M_t$:
\begin{align}
M_t &\geq B_t\qquad \forall t
\\
\impliedby 
\frac{(e-1)\beta}{4e}\expo{2n^pe^{pt}-t} &\geq \frac{\tau^4\parens{ne^t}^{2-2p}}{8\cdot9984^2\ln\parens{\frac{1664e\ln\parens{\frac{208}{\tau}}}{(e-1)\tau\beta}}}\qquad \forall t
% \\
% \iff
% \frac{8\cdot9984^2\ln\parens{\frac{1664e\ln\parens{\frac{208}{\tau}}}{(e-1)\tau\beta}}}{\tau^4\parens{ne^t}^{2-2p}}\expo{2n^pe^{pt}-t} &\geq \frac{4e}{(e-1)\beta}
\\
\iff 
\inf_{t\in\mathbb{N}}2n^pe^{pt}-(3-2p)t-(2-2p)\ln n &\geq \ln\parens{\frac{e\tau^4}{2\cdot9984^2(e-1)\beta\ln\parens{\frac{1664e\ln\parens{\frac{208}{\tau}}}{(e-1)\tau\beta}}}}
\end{align}
By Lemma \ref{lem:sup-bound-for-bt-leq-mt}, the infimum can be lower bounded by $\ln n - \frac{3-2p}{p}\ln\frac{3-2p}{2ep}$ when $n \geq \parens{\frac{3-2p}{2p}}^{1/p}$. Therefore, $\forall t$ $B_t \leq M_t$ is implied by
\begin{equation} \label{eq:n-lowerbound-3}
    n \geq \max\set{\frac{e\tau^4\parens{\frac{3-2p}{2ep}}^{\frac{3-2p}{p}}}{2\cdot9984^2(e-1)\beta\ln\parens{\frac{1664e\ln\parens{\frac{208}{\tau}}}{(e-1)\tau\beta}}},\ \parens{\frac{3-2p}{2p}}^{1/p}}
    \impliedby
    n \geq \parens{\frac{3-2p}{2p}}^{\frac{3-2p}{p}}
\end{equation}

Therefore, in order to satisfy the hypotheses of Theorem \ref{thm:TO}, we require from \eqref{eq:n-lowerbound-1}, \eqref{eq:n-lowerbound-2}, and \eqref{eq:n-lowerbound-3} that
\begin{equation}
    n \geq \max\set{\frac{21632\ln\parens{\frac{6656e^2\ln\parens{\frac{208}{\tau}}}{(e-1)\tau\beta}}}{\tau^2},\ \parens{\frac{8\cdot9984^2\ln\parens{\frac{1664e\ln\parens{\frac{208}{\tau}}}{(e-1)\tau\beta}}}{\tau^4}+\frac{4\cdot9984^2}{(1-p)\tau^4}}^\frac{1}{2-2p},\ \parens{\frac{3-2p}{2p}}^{\frac{3-2p}{p}}}
\end{equation}
Generally speaking, the first term will dominate when $p$ is relatively far from both zero and one, the second term will dominate as $p \to 1$, and the third term will dominate when $p \to 0$.

By Theorem \ref{thm:TO}, in round $t$, all answers returned by \textsc{Thresholdout} satisfy $\abs{a_i - \E{q_i}} \leq \tau$ with probability $1-\beta_t$. Therefore, 
\begin{equation}
\Prob{\exists i\ \abs{a_i - \E{q_i}} > \tau} \leq \sum_{t=0}^{\infty}\beta_t = \frac{(e-1)\beta}{e}\sum_{t=0}^{\infty} e^{-t} = \beta
\end{equation}
\end{proof}

\begin{restatable}{theorem}{TOsustainability}[Sustainability]\label{thm:TOsustainability}
For any $\tau,\beta,p \in (0,1)$ and any sequence of queries, \textsc{EverlastingTO} charges enough for queries such that it can always afford to buy new datasets, excluding the initial budget.
\end{restatable}
\begin{proof}
The $t^\textrm{th}$ instance of \textsc{Thresholdout} halts only after it has either answered $M_t$ total queries or at least $B_t$ queries with $o = \top$. In the first case, the total revenue is at least $M_t\cdot\frac{2N_{t+1}}{M_t} = 2N_{t+1}$ and in the latter case, the total revenue is at least $B_t \cdot \frac{2N_{t+1}}{B_t} = 2N_{t+1}$. Either way, it can affort to buy $S_{t+1},T_{t+1}$, which have size $N_{t+1}$ each.
\end{proof}

\begin{restatable}{theorem}{TOnacost}[Non-Adaptive Cost]\label{thm:TOnacost}
For any $\tau,\beta,p \in (0,1)$, a sufficiently large initial budget, and any sequence of querying rules, the total cost, $\Pi$, to a non-adaptive user who makes $M$ queries to \textsc{EverlastingTO} satisfies
\[
\Prob{\Pi > 2e^3 \ln^{1/p}\parens{\frac{eM}{(e-1)\beta}}} \leq \beta 
\]
\end{restatable}
\begin{proof}
By Theorem \ref{thm:TO}'s guarantee on \textsc{Thresholdout} and a union bound over all $t$, all non-adaptive queries are answered with $o = \perp$ with probability at least $1-\sum_{t=0}^\infty \beta_t = 1-\beta$. For the rest of the proof, we condition on this event. 

First, observe that the cost of a query with $o = \perp$ is non-increasing over time, so the cost of any $M$ non-adaptive queries is no more than the cost of making the \emph{first} $M$ non-adaptive queries. Let $T$ be the round in which the $M^{\textrm{th}}$ non-adaptive query is made if no adaptive queries are made. 

Let $\Pi$ be the total amount paid. This is at most the total number of samples used in rounds $1$ through $T+1$, i.e.
\begin{equation}\label{eq:nacostub}
\Pi \leq \sum_{t=1}^{T+1} 2N_t = 2n\sum_{t=1}^{T+1} e^t \leq 2ne^{T+2}
\end{equation}
Furthermore, the total number of queries made satisfies
\begin{equation}
M \geq M_{T-1} = \beta_{T-1}\expo{2N_{T-1}^p}
\end{equation}
which implies
\begin{equation}\label{eq:namlb}
\ln\parens{\frac{eM}{(e-1)\beta}} \geq 2N_{T-1}^p - (T-1) \geq N_{T-1}^p = n^pe^{p(T-1)}
\end{equation}
where we use the fact that $n \geq \parens{1/p}^{1/p}$ (see proof of Theorem \ref{thm:TOvalidity}) which implies $N_{T-1}^p = n^p e^{p(T-1)} \geq \frac{e^{p(T-1)}}{p} \geq \frac{p(T-1)}{p} = T-1$.
Combining \eqref{eq:nacostub} and \eqref{eq:namlb},
\begin{equation}
\Pi \leq 2ne^{T+2} \leq 2e^3 \ln^{1/p}\parens{\frac{eM}{(e-1)\beta}}\qedhere
\end{equation}
\end{proof}

\begin{restatable}{theorem}{TOacost}[Adaptive Cost]\label{thm:TOacost}
For any $\tau,\beta \in (0,1)$, $p \in (0,\frac{2}{3})$, a sufficiently large initial budget, and any sequence of querying rules, the total cost, $\Pi$, to a user who makes $B$ potentially adaptive queries to \textsc{EverlastingTO} satisfies
\[
\Prob{\Pi \leq 2e^2 \parens{\frac{8\cdot9984^2eB\ln\parens{\frac{1664\ln\parens{\frac{208}{\tau}}}{(e-1)\tau\beta}}}{\tau^4}}^{\frac{1}{2-3p}}} = 1 
\]
\end{restatable}
\begin{proof}
First, observe that the cost of a query is non-increasing over time, so the cost of any $B$ adaptive queries is no more than the cost of making the \emph{first} $B$ adaptive queries. Furthermore, adaptive queries may be answered with either $\top$ or $\perp$, but since $B_t \leq M_t$ $\forall t$, the cost of an adaptive query in round $t$ is no more than $\frac{2N_{t+1}}{B_t}$. Let $T$ be the round in which the $B^{\textrm{th}}$ adaptive query is made. 
Let $\Pi$ be the total amount paid. This is at most the total number of samples used in rounds $1$ through $T+1$, i.e.
\begin{equation}\label{eq:acostub}
\Pi \leq \sum_{t=1}^{T+1} 2N_t = 2n\sum_{t=1}^{T+1} e^t \leq 2ne^{T+2}
\end{equation}
Furthermore, the total number of adaptive queries is
\begin{align}
B \geq \sum_{t=0}^{T-1} B_t &= \sum_{t=0}^{T-1} \frac{\tau^4N_t^{2-2p}}{8\cdot9984^2\ln\parens{\frac{1664\ln\parens{\frac{208}{\tau}}}{\tau\beta_t}}} \\
&\geq 
\frac{\tau^4}{8\cdot9984^2\parens{T-1 + \ln\parens{\frac{1664e\ln\parens{\frac{208}{\tau}}}{(e-1)\tau\beta}}}}
\sum_{t=0}^{T-1} N_t^{2-2p} 
\\
&=
\frac{\tau^4n^{2-2p}}{8\cdot9984^2\parens{T + \ln\parens{\frac{1664\ln\parens{\frac{208}{\tau}}}{(e-1)\tau\beta}}}}
\sum_{t=0}^{T-1} e^{t(2-2p)} 
\\
&\geq 
\frac{\tau^4n^{2-2p}\parens{e^{T(2-2p)} - 1}}{8\cdot9984^2T\ln\parens{\frac{1664\ln\parens{\frac{208}{\tau}}}{(e-1)\tau\beta}}}
\\
&\geq 
\frac{\tau^4n^{2-2p}e^{T(2-2p)-1}}{8\cdot9984^2T\ln\parens{\frac{1664\ln\parens{\frac{208}{\tau}}}{(e-1)\tau\beta}}} \label{eq:acostintermed}
% \\
% &\geq
% \frac{\tau^4n^{2-2p}\parens{e^{T(2-2p)} - 1}}{8\cdot9984^2e^T\ln\parens{\frac{1664\ln\parens{\frac{208}{\tau}}}{(e-1)\tau\beta}}}
% \\
% &\geq 
% \frac{\tau^4n^{2-2p}\parens{e^{T(2-2p)} - 1}}{8\cdot9984^2e^T\ln\parens{\frac{1664\ln\parens{\frac{208}{\tau}}}{(e-1)\tau\beta}}}
\end{align}
Where in the last inequality we used that $p < \frac{2}{3}$ so $e^{T(2-2p)} - 1 \geq e^{T(2-2p)-1}$.
Since $n \geq \parens{1/p}^{1/p}$ (see proof of Theorem \ref{thm:TOvalidity}), it is also the case that $n^pe^{pT} \geq T$. Picking up from \eqref{eq:acostintermed}, we have
\begin{equation}
\frac{8\cdot9984^2B\ln\parens{\frac{1664\ln\parens{\frac{208}{\tau}}}{(e-1)\tau\beta}}}{\tau^4} 
\geq 
\frac{n^{2-2p}e^{T(2-2p)-1}}{n^pe^{pT}} 
=
n^{2-3p}e^{T(2-3p)-1}
\end{equation}
thus
\begin{equation}\label{eq:ablb}
ne^T \leq \parens{\frac{8\cdot9984^2eB\ln\parens{\frac{1664\ln\parens{\frac{208}{\tau}}}{(e-1)\tau\beta}}}{\tau^4}}^{\frac{1}{2-3p}}
\end{equation}

Combining \eqref{eq:acostub} and \eqref{eq:ablb}, we get that
\begin{equation}
\Pi \leq 2ne^{T+2} \leq 2e^2 \parens{\frac{8\cdot9984^2eB\ln\parens{\frac{1664\ln\parens{\frac{208}{\tau}}}{(e-1)\tau\beta}}}{\tau^4}}^{\frac{1}{2-3p}} \qedhere
\end{equation}
\end{proof}

To expand on the guarantees of Theorems \ref{thm:TOnacost} and \ref{thm:TOacost}, $p$ is a parameter of the algorithm that can be chosen roughly in the range $(0,1)$. These theorems could be stated instead in terms of the quantity $a = 1/p$, which lies generally in the range $(1,\infty)$. In this case, a sequence of $M$ non-adaptive queries would cost (with high probability) at most $\mathcal{O}\parens{\ln^a M}$, and a sequence of $M$ adaptive queries would cost at most $\mathcal{O}\parens{B^{\frac{a}{2a-3}}}$. That is, when $a$ is near $1$, we approach the optimal $\log M$ cost for non-adaptive queries at the expense of a very large (exploding) cost of adaptive queries. On the other hand, as we made $a$ very large, we approach the optimal $\sqrt{M}$ cost for adaptive queries at the expense of more expensive polylog cost for non-adaptive queries. In this way, the parameter $p$ trades off between placing the burden of adaptivity directly on the adaptive queries themselves and spreading it out over potentially non-adaptive queries too.

\begin{lemma}\label{lem:sup-bound-for-bt-geq-1}
For any $\beta,\tau,p \in (0,1)$,
\[
\sup_{t\in\mathbb{N}} e^{-t} \parens{\frac{8\cdot9984^2\parens{t + \ln\parens{\frac{1664e\ln\parens{\frac{208}{\tau}}}{(e-1)\tau\beta}}}}{\tau^4}}^{\frac{1}{2-2p}} \leq \parens{\frac{8\cdot9984^2}{\tau^4}\parens{\ln\parens{\frac{1664e\ln\parens{\frac{208}{\tau}}}{(e-1)\tau\beta}}+\frac{1}{2-2p}}}^\frac{1}{2-2p}
\]
\end{lemma}
\begin{proof}
For brevity, let $a := \frac{8\cdot9984^2}{\tau^4}$, let $b := \ln\parens{\frac{1664e\ln\parens{\frac{208}{\tau}}}{(e-1)\tau\beta}}$, and let $c = \frac{1}{2-2p}$, note that $a,b,c > 0$. We are thus interested in upper bounding $\sup_{t\in\mathbb{N}} e^{-t}\parens{at + ab}^{c}$. First, 
\begin{equation}
    \frac{d}{dt} e^{-t}\parens{at + ab}^{c} = ace^{-t}\parens{at + ab}^{c-1} - e^{-t}\parens{at + ab}^c
\end{equation}
and
\begin{equation}
    ace^{-t}\parens{at + ab}^{c-1} - e^{-t}\parens{at + ab}^c = 0 \iff t = c - b \textrm{ or } t = -b \textrm{ or } t \to \infty
\end{equation}
Since we are only optimizing over $t \in \mathbb{N}$ and $b > 0$, we do not need to consider the critical point $t = -b$. Furthermore, 
\begin{equation}
    \left.\frac{d^2}{dt^2} e^{-t}\parens{at + ab}^{c}\right|_{t=c-b} = -\frac{1}{c}(ac)^ce^{b-c} < 0
\end{equation}
Therefore, the critical point at $t = c-b$ is a local maximum. Therefore, the only points we need to consider are when $t = 0$, $t \to \infty$, and $t = c-b$ if $c \geq b$.
\begin{equation}
    \sup_{t\in\mathbb{N}} e^{-t}\parens{at + ab}^{c} \leq 
    \begin{cases}
        (ab)^c & b > c \\
        \max\set{\parens{ab}^{c},e^{b-c}(ac)^c} & c \geq b
    \end{cases}
    \leq a^c(b+c)^c
\end{equation}
which completes the proof.
\end{proof}

\begin{lemma} \label{lem:sup-bound-for-bt-leq-mt}
For any $p \in (0,1)$ and $n \geq 1$
\[
\inf_{t\in\mathbb{N}}2n^pe^{pt}-(3-2p)t-(2-2p)\ln n \geq \min\set{\ln n - \frac{3-2p}{p}\ln\frac{3-2p}{2ep},\ 2n^p - (2-2p)\ln n}
\]
and the first term is the minimizer when $n \geq \parens{\frac{3-2p}{2p}}^{1/p}$
\end{lemma}
\begin{proof}
First, note that this is a convex function in $t$ and
\begin{equation}
    \frac{d}{dt} 2n^pe^{pt}-(3-2p)t-(2-2p)\ln n = 2pn^pe^{pt} - 3 + 2p
\end{equation}
and
\begin{equation}
    2pn^pe^{pt} - 3 + 2p = 0 \iff t = \frac{1}{p}\ln\frac{3-2p}{2p} - \ln n
\end{equation}
Therefore, if $\frac{1}{p}\ln\frac{3-2p}{2p} - \ln n \geq 0$ then
\begin{equation}
    \inf_{t\in\mathbb{N}}2n^pe^{pt}-(3-2p)t-(2-2p)\ln n \geq \ln n - \frac{3-2p}{p}\ln\frac{3-2p}{2ep} 
\end{equation}
Otherwise, if $\frac{1}{p}\ln\frac{3-2p}{2p} - \ln n < 0$
\begin{equation}
    \inf_{t\in\mathbb{N}}2n^pe^{pt}-(3-2p)t-(2-2p)\ln n \geq 2n^p - (2-2p)\ln n 
\end{equation}
Thus,
\begin{equation}
    \inf_{t\in\mathbb{N}}2n^pe^{pt}-(3-2p)t-(2-2p)\ln n \geq \min\set{\ln n - \frac{3-2p}{p}\ln\frac{3-2p}{2ep},\ 2n^p - (2-2p)\ln n }
\end{equation}
\end{proof}

\section{Relevant Results in Differential Privacy}\label{sec:dpfactoids}
Here, we state without proof definitions and results from other work which we use in the proof of Lemma \ref{lem:adaptive_S_close_to_pop}.
\begin{definition}\label{def:dp}
A randomized algorithm $\M : \X^* \mapsto \mathcal{Y}$ is $(\epsilon,\delta)$-differentially private if for all $E\subseteq \mathcal{Y}$ and all datasets $S,S' \in \X^*$ differing in a single element:
\[
\Prob{\M(S) \in E} \leq e^\epsilon \Prob{\M(S') \in E} + \delta .
\]
\end{definition}
\begin{proposition}[\cite{NissimS15,BassilyNSSSU16}]\label{prop:dp_generalizes}
Let $\M$ be an $(\epsilon,\delta)$-differentially private algorithm that outputs a function from $\mathcal{X}$ to $[0,1]$. For a random variable $S\sim\D^n$ we{} let $q = \M(S)$. Then for $n \geq 2\ln(8/\delta)/\epsilon^2$,
\[
\Prob{\abs{\ee{S}{q} - \E{q}} \geq 13\epsilon} \leq \frac{2\delta}{\epsilon}\ln\parens{\frac{2}{\epsilon}}
.\]
\end{proposition}

\begin{definition}[Definition 1.1 \cite{BunSteinke16}]\label{prop:zcdp}
A randomized mechanism $M:\X^n\rightarrow\mathcal{Y}$ is $\rho$-zero-concentrated differentially private (henceforth $\rho$-zCDP) if, for all $S,S'\in\X^n$ differing on a single entry and all $\alpha \in (1,\infty)$,
\[
D_\alpha\parens{\M(S) || \M(S')} \leq \rho\alpha,
\]
where $D_\alpha\parens{\M(S) || \M(S')}$ is the $\alpha$-R\'enyi divergence between the distribution of $\M(S)$ and $\M(S')$.
\end{definition}

\begin{proposition}[Proposition 1.6 \cite{BunSteinke16}]\label{prop:gaussian_mechanism}
Let $q$ be a statistical query. Consider the mechanism $\M:\X^n\rightarrow\R$ that on input $S$, releases a sample from $\N(\ee{S}{q},\sigma^2)$. Then $\M$ satisfies $\frac{1}{2n^2\sigma^2}$-zCDP.
\end{proposition}

\begin{proposition}[Lemma 1.7 \cite{BunSteinke16}]\label{prop:zcdp_composition}
Let $\M:\X^n\rightarrow\mathcal{Y}$ and $\M':\X^n\rightarrow\mathcal{Z}$ be randomized algorithms. Suppose $\M$ satisfies $\rho$-zCDP and $\M'$ satisfies $\rho'$-zCDP. Define $\M'':\X^n\rightarrow\mathcal{Y}\times\mathcal{Z}$ by $\M''(x) = \parens{\M(x), \M'(x)}$. Then $\M''$ satisfies $(\rho+\rho')$-zCDP.
\end{proposition}

\begin{proposition}[Proposition 1.3 \cite{BunSteinke16}]\label{prop:zcdp_to_dp}
If $\M$ provides $\rho$-zCDP, then $\M$ is $\parens{\rho + 2\sqrt{\rho\ln(1/\delta)},\delta}$-differentially private for any $\delta > 0$.
\end{proposition}

\end{document}